\documentclass[12pt, a4paper]{article}

\usepackage{graphicx}
\usepackage{amsmath, amssymb, amsthm, enumerate}
\usepackage{graphics}
\usepackage{graphicx}
\usepackage{tabularx,multirow,bigdelim}

\usepackage{cite}
\usepackage{amsmath,amssymb,amsthm,enumerate,mathrsfs}
\usepackage{algorithm}
\usepackage{algpseudocode}

\newtheorem{Theorem}{Theorem}[section]

\newtheorem{Proposition}{Proposition}[section]
\theoremstyle{definition}
\newtheorem{Definition}{Definition}[section]

\numberwithin{equation}{section}

\begin{document}
\makeatletter
\begin{center}
%%%%Performance Optimization of a Dynamic Channel Bonding Strategy in Cognitive Radio Networks
\large{\bf Coordinate Descent for MCP/SCAD Penalized Least Squares Converges Linearly}
\end{center}\vspace{5mm}
\begin{center}
\textsc{Yuling Jiao \footnote{ School of Mathematics and Statistics, and
Hubei Key Laboratory of Computational Science, Wuhan University, Wuhan 430072, PR China
(yulingjiaomath@whu.edu.cn)}, Dingwei Li\footnote{ School of Mathematics and Statistics, Wuhan University, Wuhan 430072, PR China (lidingv@whu.edu.cn)}, Min Liu\footnote{ School of Mathematics and Statistics, Wuhan University, Wuhan 430072, PR China
(mliuf@whu.edu.cn)}, Xiliang Lu \footnote{Corresponding author, School of Mathematics and Statistics, and
Hubei Key Laboratory of Computational Science, Wuhan University, Wuhan 430072, PR China
(xllv.math@whu.edu.cn)}}
\end{center}
\vspace{2mm}
\noindent\begin{minipage}{14cm}
{\bf Abstract:}
Recovering sparse signals from observed data is an important topic   in signal/imaging processing, statistics and machine learning. Nonconvex penalized least squares have been attracted a lot of attentions since they enjoy nice statistical properties. Computationally,   coordinate descent (CD) is a workhorse for minimizing the nonconvex penalized least squares criterion  due to its  simplicity and  scalability.  In this work,  we prove the linear convergence rate to CD for solving MCP/SCAD penalized least squares problems.
\end{minipage}
\\[5mm]

\hbox to14cm{\hrulefill}\par
%%%%%%%%%%%%%%%%%%%%%%%%%%%%%%%%%%%%%%%%%%%%%%%%%%%%%%%%%%%%%%%%%%%%%%%
%%%%%%%%%%%%%

\section{Introduction}
Considering the sparse linear estimation   problem
\begin{equation}\label{LinearEquationWithNoise}
b = Ax^{\ast}+\xi,
\end{equation}
where the vector $x^{\ast}\in\mathbb{R}^{p}$ denotes the sparse regression coefficient or sparse signal to be recovered, the vector $\xi\in\mathbb{R}^{n}$ is the random error term, and the design matrix $A\in\mathbb{R}^{n\times p}$  with $n\ll p$ describing the system response mechanism. Throughout, we assume the matrix $A$ has normalized column vectors $\{A_{i}\}$, i.e., $\|A_{i}\|_{2}=1$ for $i=1,\dots,p$.
 %we have $n<<p$, which is considered as severely underdetermined (and ill-posed). Hence, it is considerably challenging to obtain a meaningful solution. The sparsity approach looks for a solution with many zero entries, and it opens a novel avenue for resolving the ill-posedness.
 The basis pursuit \cite{chen2001atomic} or lasso \cite{tibshirani1996regression}
\begin{equation}\label{Lasso}
\min_{x\in\mathbb{R}^{p}}F(x)=\frac{1}{2}\|Ax-b\|^{2}_{2}+\lambda\|x\|_{1},
\end{equation}
%where $\|\cdot\|_{1}$ denotes the $\ell^{1}$-norm of a vector, and $\lambda >0$ is a regularization parameter.
is  an widely used  sparse recovery model.
  %Since its reliable ability to reconstruct signal with sparsity and there are plenty of efficient methods, \eqref{Lasso} has achieved considerable success and won popularity.
 The minimizers of Lasso  \eqref{Lasso} enjoy attractive statistical properties \cite{candes2005decoding, meinshausen2006high, zhao2006on}.
The convexity of the problem \eqref{Lasso} allows designing fast and global convergent  algorithms.
 %e.g., gradient projection methods and coordinate descent algorithms,
 see \cite{tropp2010computational} for an overview.
However,
%there are some drawbacks while convex models gain globally convergent minimization and admit efficient numerical solution. For example, convex model require more restrictive conditions on the design matrix $A$ and more data in order to recover exactly the signal than nonconvex ones, e.g., bridge penalty \cite{donoho2012sparse, chartrand2007restricted, sun2012recovery}; and it
 the Lasso estimator tends to produce biased estimates for large coefficients \cite{zhang2008the}, and hence lacks oracle property
\cite{fan2001variable, fan2004nonconcave}. Several nonconvex penalty functions has been proposed to remedy this including the MCP \cite{zhang2010Nearly} and SCAD \cite{fan2001variable, fan2004nonconcave}.

Consider the following nonconvex optimization problem
\begin{equation}\label{nonconvex model}
\min_{x\in\mathbb{R}^{p}}F(x)=\frac{1}{2}\|Ax-b\|^{2}_{2}+\sum_{i=1}^{p}\rho_{\lambda,\tau}(x_{i}),
\end{equation}
where $\rho_{\lambda,\tau}$ is a non-convex penalty, $\lambda>0$ is a regularization parameter, and $\tau\geq 0$ controls the degree of concavity of penalty.
The nonconvex function $\rho_{\lambda,\tau}$
 satisfies the requirements that it is singular at the origin in order to achieve sparsity and its derivative vanishes for large values so as to ensure unbiasedness.  For SCAD, it is defined for $\tau > 2$ via
\begin{equation}\label{integralscad}
\rho_{\lambda, \tau}(t) = \lambda\int_{0}^{|t|}\min\left(1,  \frac{\max\left(0, \lambda\tau-|s|\right)}{\lambda(\tau-1)}\right)\mbox{d}s
\end{equation}
and computing the integral explicitly yields the expression in Table \ref{table:Nonconvex penalty}. Further, variable selection consistency and asymptotic estimation efficiency were studied in \cite{fan2004nonconcave}. MCP was devised in the same spirit as SCAD which is defined as
\begin{equation}\label{integralmcp}
\rho_{\lambda, \tau}(t) = \lambda\int_{0}^{|t|}\max\left(0, 1-\frac{|s|}{\lambda\tau}\right)\mbox{d}s.
\end{equation}
MCP minimizes the maximum concavity $\sup_{0<t_1<t_2}\frac{\left( \rho'_{\lambda,\tau}(t_{1}) - \rho'_{\lambda,\tau}(t_{2}) \right)}{t_{2} - t_{1}}$ to satisfy unbiasedness and feature selection constraints: $\rho_{\lambda, \tau}'(t) = 0$ for any $|t|\geq \lambda\tau$ and $\rho_{\lambda\tau}'(0^{\pm}) = \pm\lambda$. The condition $\tau>1$ ensures the well-posedness of the thresholding operator \cite{zhang2010Nearly}. The gradient functions of SCAD and MCP are
\begin{equation}\label{gradscad}
\rho_{\lambda, \tau}'(t) =
\left \{\begin{array}{rcl}
&0&,\ |t|\geq \lambda\tau,\\
&\frac{\lambda\tau-\frac{1}{2}|t|}{\tau-1}&,\ \lambda<|t|<\lambda\tau,\\
&\lambda&,\ |t|\leq \lambda\tau
\end{array}\right.
\nonumber
\end{equation}
and
\begin{equation}\label{gradmcp}
\rho_{\lambda, \tau}'(t) =
\left \{\begin{array}{rcl}
&\lambda - \frac{|t|}{\lambda\tau}&,\ |t|< \lambda\tau,\\
&0&,\ |t|\geq \lambda\tau
\end{array}\right.
\nonumber
\end{equation}
respectively. We summarize the function $\rho_{\lambda, \tau}$ corresponding to Lasso, SCAD,  MCP and their  thresholding functions in Table \ref{table:Nonconvex penalty}. We plot the  Lasso, MCP, SCAD penalties, derivative of  these  penalties and their   thresholding functions in Figure
\ref{fig1}.
%The nonconvexity and nonsmoothess of the penalty $\rho_{\lambda,\tau}$ poses significant challenges on mathematical analysis and efficient numerical solutions.
%Below we mainly consider two popular nonconvex penalty function of
\begin{table}[H]\label{table:Nonconvex penalty}
\centering
\caption{Nonconvex penalty $\rho_{\lambda, \tau}(t)$ and the thresholding operators $\mathcal{S}^{\rho}_{\lambda,\tau}$}
\begin{tabular}{lll}
\hline
penalty             & $\rho_{\lambda, \tau}(t)$ & $\mathcal{S}^{\rho}_{\lambda,\tau}(c)$ \\ \hline
LASSO               & $\lambda|t|$              & $\text{sgn}(c)\max\{|c|-\lambda, 0 \}$ \\
SCAD $(\tau >2)$    & $\left\{\begin{array}{ll}\frac{\lambda^{2}(\tau+1)}{2} & |t|>\lambda \tau \\ \frac{\lambda \tau|t|-\frac{1}{2}\left(t^{2}+\lambda^{2}\right)}{\tau-1} & \lambda<|t| \leq \lambda \tau \\ \lambda|t| & |t| \leq \lambda\end{array}\right.$
 &
$\left\{\begin{array}{ll}0 & |v| \leq \lambda \\ \operatorname{sgn}(v)(|v|-\lambda) & \lambda<|v| \leq 2 \lambda \\ \operatorname{sgn}(v) \frac{(\tau-1)|v|-\lambda \tau}{\tau-2} & 2 \lambda<|v| \leq \lambda \tau \\ v & |v|>\lambda \tau\end{array}\right.$ \\
MCP $(\tau >1)$     & $\left\{\begin{array}{ll}\lambda\left(|t|-\frac{t^{2}}{2 \lambda \tau}\right) & |t|<\tau \lambda \\ \frac{\lambda^{2} \tau}{2} & |t| \geq \tau \lambda\end{array}\right.$
&
$\left\{\begin{array}{ll}0 & |v| \leq \lambda \\ \operatorname{sgn}(v) \frac{\tau(|v|-\lambda)}{\tau-1} & \lambda < |v| \leq \lambda \tau \\ v & |v| > \lambda \tau\end{array}\right.$
\\ \hline
\end{tabular}
\end{table}

\begin{figure}[H]\label{fig1}
%\centering
\includegraphics[width=1.8in]{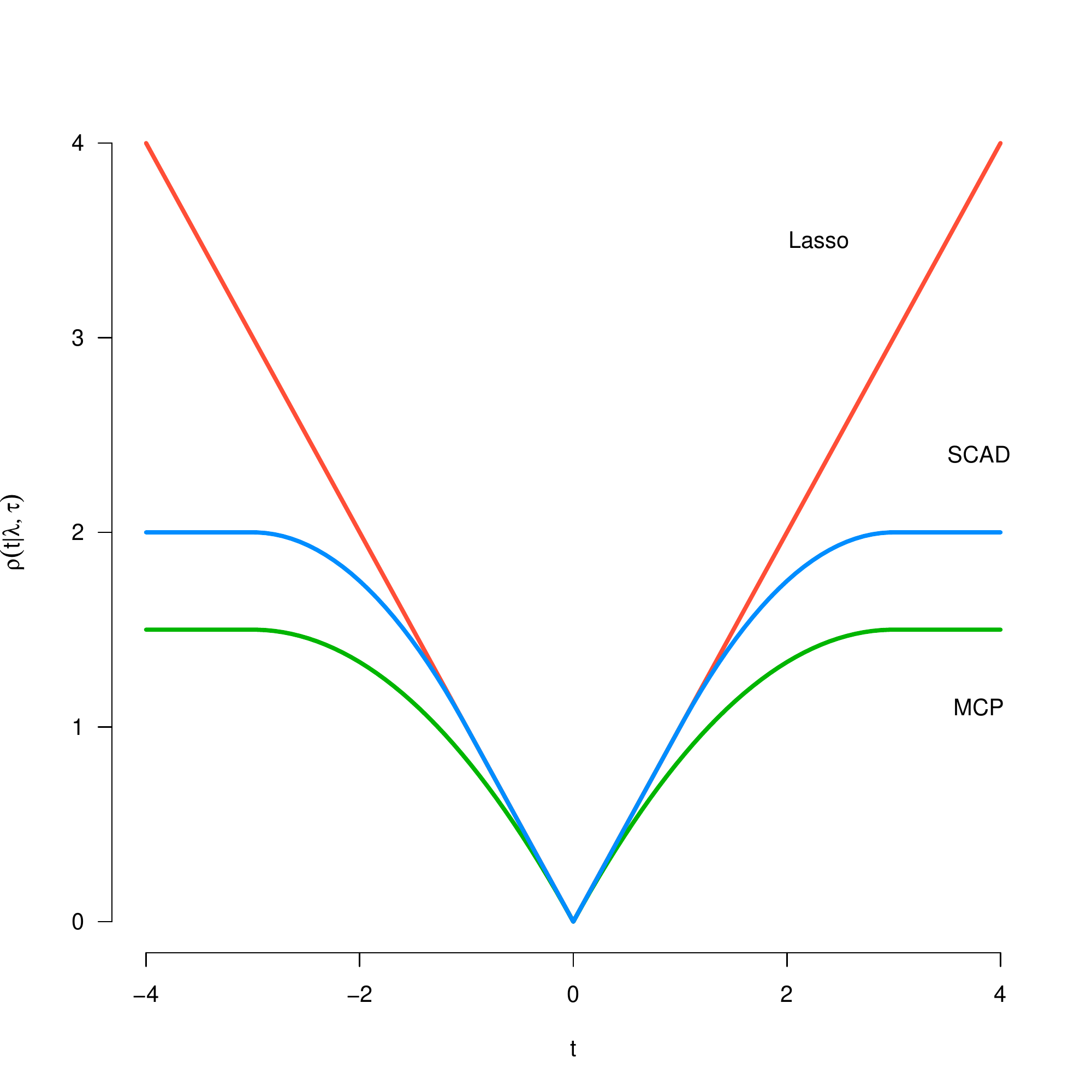}
\includegraphics[width=1.8in]{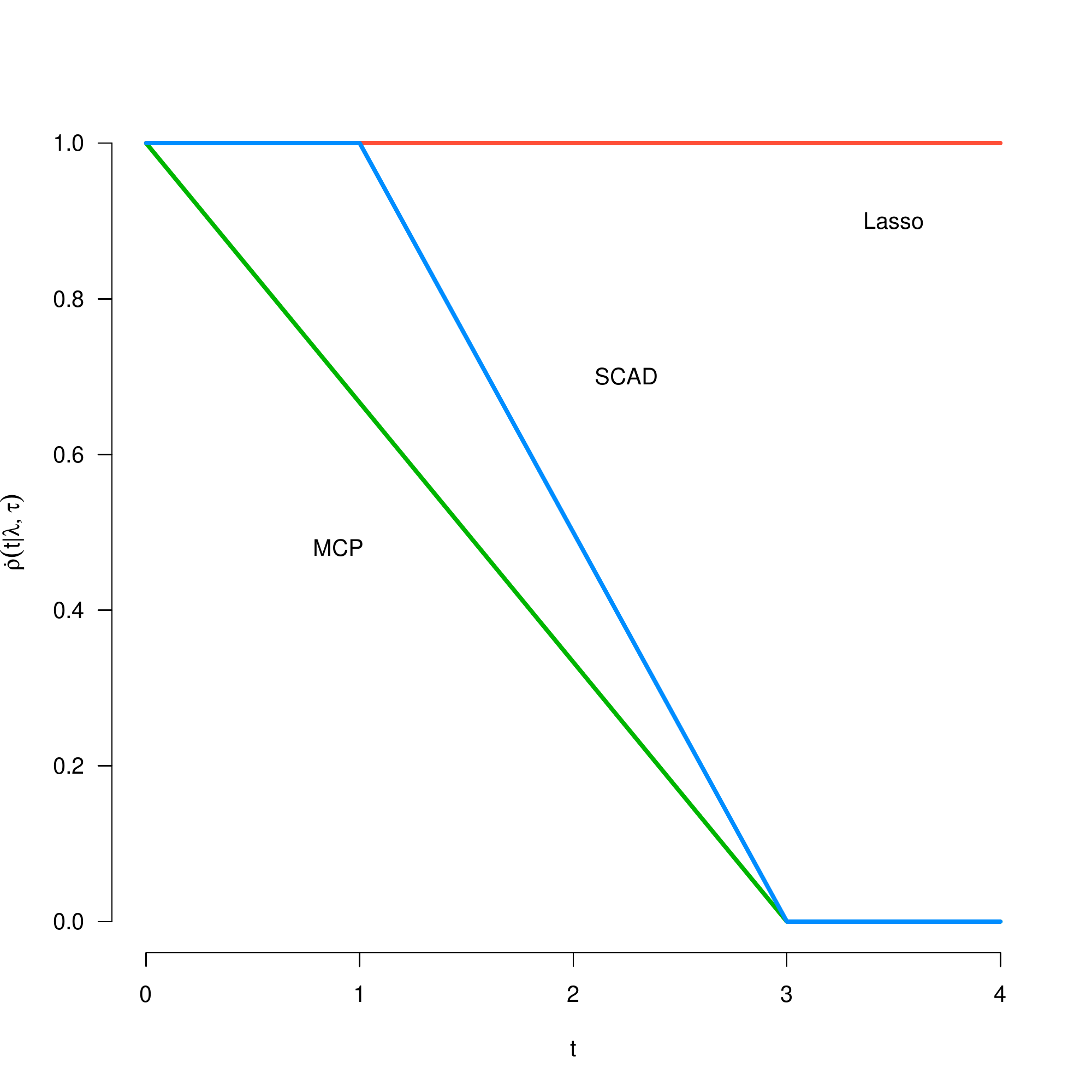}
\includegraphics[width=1.8in]{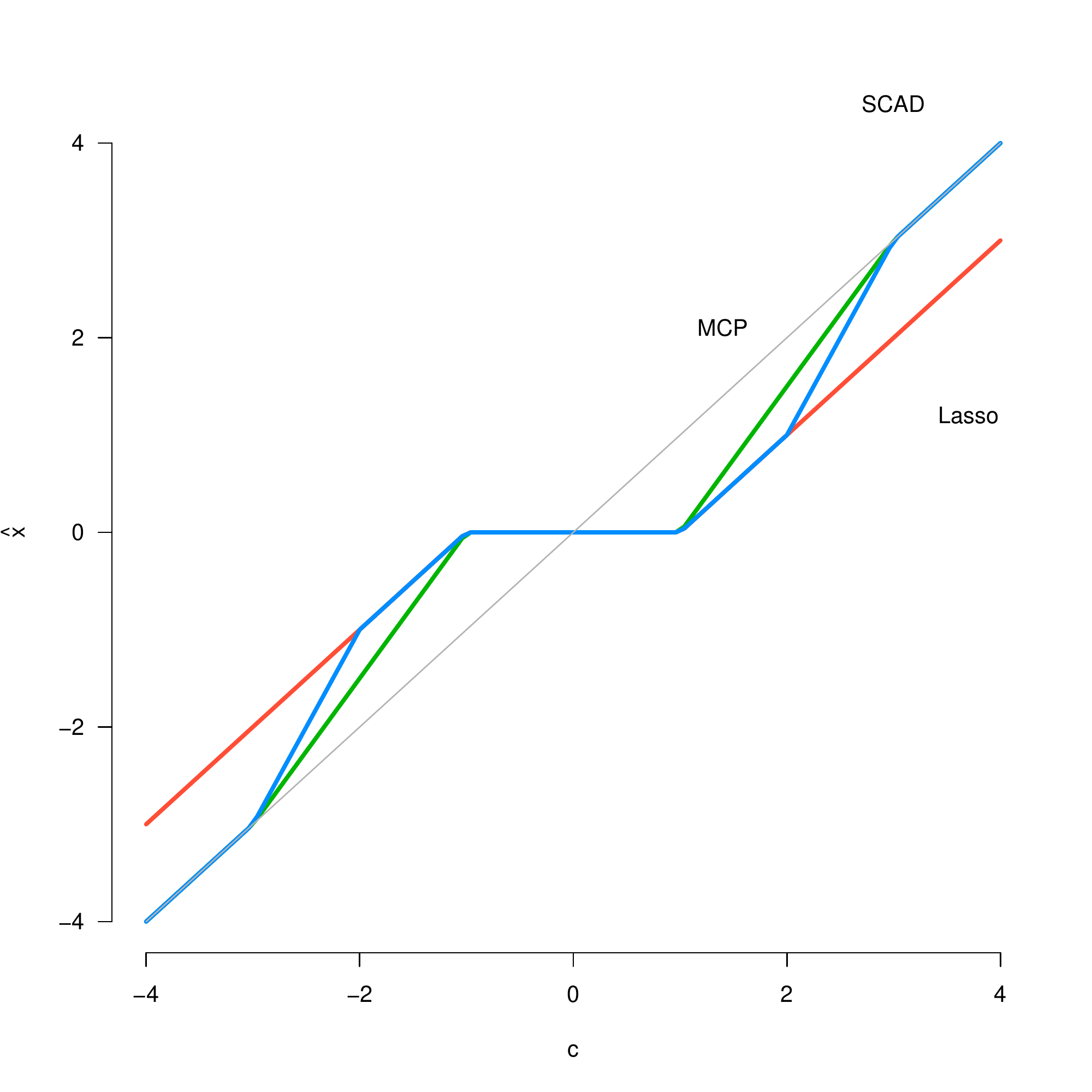}
\caption{Lasso, SCAD, MCP  penalties, derivative of  these penalties and their  thresholding functions.}
\end{figure}

%\subsection{Main work}
%The following description will center at the topics of designing efficient algorithm and its theoretical guarantee. Coordinate descent algorithm solves optimization problems by successively performing approximate minimization along coordinate directions or coordinate hyperplanes, which has gained great popularity due to its usefulness in data analysis, machine learning, and other areas of current interest.
%\subsection{Related works}
The nonconvexity and nonsmoothess of the SCAD and MCP penalty  poses  challenge for solving (\ref{nonconvex model}).
Several efforts has been made to handle this including local quadratic approximation (LQA)  \cite{fan2001variable}, local linear approximation (LLA) \cite{zou2008one} and multi-stage convex relaxation \cite{zhang2010analysis}, coordinate descent (CD)  in either Jacobi \cite{she2009thresholding} or Gauss-Seidel \cite{mazumder2011sparsenet, breheny2011coordinate} fashion.
Among the above mentioned numerical methods, coordinate descent proposed in   \cite{mazumder2011sparsenet, breheny2011coordinate} became a popular solver in statistical communities due to its simplicity and scalability.
 %Theoretically,
  Numerical experiments in \cite{mazumder2011sparsenet,breheny2011coordinate} demonstrates fast  convergence of CD for SCAD and MCP. However,
  the convergence analysis of  CD is fall behind its excellent numerical performance. Indeed, in \cite{mazumder2011sparsenet, breheny2011coordinate} they showed any cluster point of the iterates is a stationary point (under the assumption that the iteration  sequence has clusters)   by using the idea  developed in \cite{tseng2001convergence}. In this paper we fill this gap by showing linear convergence rate of CD for solving (\ref{nonconvex model}).

The rest of the paper are organized as follows. In  section 2, we prove the linear convergence rate of CD.
We give the conclusion in Section 3.

\section{Convergence rate analysis of CD }\label{conv}

\subsection{Coordinate descent}\label{cda}
%\subsection{Algorithm for model \eqref{nonconvex model} with MCP penalty}
In this section, we recall the CD algorithm \cite{mazumder2011sparsenet, breheny2011coordinate} for (\ref{nonconvex model}) with SCAD and MCP penalties.
The objective function  reads
\begin{equation}
F(x)=\frac{1}{2}\|Ax-b\|^{2}_{2}+\sum_{i=1}^{p}\rho_{\lambda,\tau}(x_{i}).
\nonumber
\end{equation}
Given the current iteration  $x^{k}$, we update $x^{k+1}$ by
\begin{equation*}
%\begin{aligned}
x_{i}^{k+1}=\arg\min_{t}F(x_{1}^{k
+1},\ldots,x_{i-1}^{k+1},t,x_{i+1}^{k},\ldots,x_{p}^{k})\ \mbox{for}\ i=1,\ldots,p.
%\\
%&=\arg\min_{t}\frac{1}{2}\|\sum_{j=1}^{i-1}x_{j}^{k+1}A_{j}+tA_{i}+\sum_{j=i+1}^{p}x_{j}^{k}A_{j}-b\|_{2}^{2}\\
%&+\sum_{j=1}^{i-1}\rho_{\lambda,\tau}(x_{j}^{k+1})+\rho_{\lambda,\tau}(t)+\sum_{j=i+1}^{p}\rho_{\lambda,\tau}(x_{j}^{k})\\
%&\mbox{with }  b_{i}^{k} = b - \sum_{j=1}^{i-1}x_{j}^{k+1}A_{j} - \sum_{j=i+1}^{p}x_{j}^{k}A_{j}\ \mbox{and  }\ p_{i}^{k} = \sum_{j=1}^{i-1}\rho_{\lambda,\tau}(x_{j}^{k+1})+\sum_{j=i+1}^{p}\rho_{\lambda,\tau}(x_{j}^{k}) \\
%&=:\arg\min_{t}\frac{1}{2}\|tA_{i}-b_{i}^{k}\|_{2}^{2}+\rho_{\lambda,\tau}(t)+p_{i}^{k}\\
%&=\arg\min_{t}\frac{1}{2}(t^{2}+\|b_{i}^{k}\|^{2}_{2}-2tA_{i}^{T}b_{i}^{k})+\rho_{\lambda,\tau}(t)+p_{i}^{k}\\
%&=\arg\min_{t}\frac{1}{2}(t-A_{i}^{T}b_{i}^{k})^{2}+\|b_{i}^{k}\|^{2}_{2}-(A_{i}^{T}b_{i}^{k})^{2}+\rho_{\lambda,\tau}(t)+p_{i}^{k}\\
%&=\arg\min_{t}\frac{1}{2}(t-A_{i}^{T}b_{i}^{k})^{2}+\rho_{\lambda,\tau}(t)\\
%&=:\arg\min_{t}\frac{1}{2}(t-C_{i}^{k})^{2}+\rho_{\lambda,\tau}(t)\nonumber
%\end{aligned}
\end{equation*}
Some algebra shows that
$$x^{k+1}_i \in \arg\min_{t} f_i(t) := \frac{1}{2}(t-c^k_i)^{2}+\rho_{\lambda,\tau}(t) $$
where
\begin{align}\label{cik}
c_{i}^{k}  = A_{i}^{T}\left(  b - \sum_{j=1}^{i-1}x_{j}^{k+1}A_{j} - \sum_{j=i+1}^{p}x_{j}^{k}A_{j} \right).
\end{align}
%It is obvious that the optimization problem
%\begin{equation}
%\min_{t\in \mathbb{R}}f(t)=\frac{1}{2}(t-c)^{2}+\rho_{\lambda,\tau}(t)
%\end{equation}
By the definition of the   thresholding operator of $\rho_{\lambda,\tau}$  in Table \ref{table:Nonconvex penalty},
$$x^{k+1}_i = S_{\lambda,\tau}^{\rho}(c^k_i), i=1,2...,p.$$
%S_{\lambda,\tau}^{\rho}(c)=\arg\min_{t\in \mathbb{R}}f(t)$, i.e.
%\begin{equation}\label{ThresholdingSCAD}
%t=S_{\lambda,\tau}^{\rho}(c)=\left\{\begin{array}{rcl}
%0 & \mbox{for} & |c|\leq \lambda, \\
%\text{sgn}(c)\frac{\tau|c|-\lambda}{\tau-1} &\mbox{for} &\lambda\leq |c|\leq \lambda\tau, \\
%c & \mbox{for} & |c|\geq \lambda\tau.
%\end{array}\right.
%\end{equation}
%for MCP penalty and
%\begin{equation}\label{ThresholdMCP}
%t=S_{\lambda,\tau}^{\rho}(c)=\left\{\begin{array}{rcl}
%0 & \mbox{for} & |c|\leq \lambda, \\
%\text{sgn}(c)(|c|-\lambda) &\mbox{for} &\lambda\leq |c|\leq 2\lambda, \\
%\text{sgn}(c)\frac{(\tau-1)|c|-\lambda\tau}{\tau-2} &\mbox{for} &2\lambda\leq |c|\leq \lambda\tau, \\
%c & \mbox{for} & |c|\geq \lambda\tau.
%\end{array}\right.
%\end{equation}
%for SCAD penalty. Further discussion about thresholding operator can turn to \cite{huang2013a}.
%
%\begin{Proposition}\label{thresholding minimizer}
%$t=S_{\lambda,\tau}^{\rho}(c)$ is a minimizer of $f(t)$. Moreover, $f(t)$ is strictly convex if penalty is either MCP or SCAD, thus $S_{\lambda,\tau}^{\rho}(c)$ is the only global minizer of $f(t)$.
%\end{Proposition}
%% \noindent
%\setlength{\parindent}{0pt}
%\begin{proof}
%this result can be easily verified or found from, e.g., \cite{2011SparseNet}\cite{huang2013a}.
%\end{proof}
%
%% \textbf{Proof}:
%
%Thus we update $x^{k+1}$ by
%\begin{equation}
%x_{i}^{k+1}=S_{\lambda,\tau}^{\rho}\left(C_{i}^{k}\right)\ for\ i=1,\ldots,p.
%\end{equation}

To sum up, we present the CD algorithm in the following algorithm

\begin{algorithm}[H]\caption{Coordinate Descent}
\label{alg:Framwork}
\begin{algorithmic}
\State Given initial point $x^{0}$, parameters $\lambda,\ \tau$
\Repeat
\State Update $x^{k+1}$ by $x_{i}^{k+1}=S_{\lambda,\tau}^{\rho}\left(c_{i}^{k}\right)\ \mathrm{for} \ i=1,\ldots,p$, with  $c_{i}^{k}$ in \eqref{cik}.
%is described in the above context;
\Until{Stop condition}\\
%\Return $x^{k+1}$;
\end{algorithmic}
\end{algorithm}

\subsection{Preliminaries on nonsmooth analysis}\label{pre}
To prove the convergence rate, we need the some tools in nonsmooth analysis including limiting subdifferential and KL property.

First we present definition of limiting  subdifferential.
%\begin{Definition}
%Let $f:\ D\rightarrow \mathbb{R}$ and $\bar{x}\in D$. $f$ is lower semi-continuous at $\bar{x}$ if for every $\varepsilon>0$, there exists $\delta>0$ such that
%\begin{equation}
%f(\bar{x})-\varepsilon < f(x)\ \mbox{for all}\ x\in \mathcal{B}(\bar{x}, \delta)\bigcap D.
%\end{equation}
%Similarly, $f$ is upper semi-continuous at $\bar{x}$ if for every $\varepsilon>0$, there exists $\delta>0$ such that
%\begin{equation}
%f(\bar{x})+\varepsilon > f(x)\ \mbox{for all}\ x\in \mathcal{B}(\bar{x}, \delta)\bigcap D.
%\end{equation}
%\end{Definition}
%Obviously, $f$ is continuous at $\bar{x}$ if and only if $f$ is lower and upper semi-continuous at this point. $f$ is lower (or upper) semi-continuous on $D$ if it is lower (or upper) at every point of $D$.
%For an extended-real-valued function $f:\ \mathbb{R}^{n}\rightarrow [-\infty, +\infty ]$, the domain is defined as $\textbf{dom}f=\{ x:\ f(x)<\infty \}$. Such a function is called proper if it is never $-\infty$ and its domain is nonempty.
Recall the definition of subdifferential at point $x$ for convex function
\begin{equation}
\partial f(x) := \left\{z\in\mathbb{R}^{n}:\ f(x) - f(z) - \langle z, x-z\rangle\geq 0 \right\}.
\end{equation}
when $f$ is non-convex, one can extend subdifferential to limiting-subdifferdntial \cite{1998Variational}.
\begin{Definition}
For a proper function $f:\ \mathbb{R}^{n}\rightarrow [-\infty, +\infty ]$, its \textbf{limiting subdifferential} at $x\in \textbf{dom}f$ is defined by
\begin{equation}
\partial_{lim} f(x) := \left\{ \nu\in\mathbb{R}^{n}:\   \exists x^{k}\stackrel{f}{\longrightarrow} x,\ \nu^{k}\rightarrow \nu\ \right\},
\end{equation}
with $ \lim_{z\rightarrow x^{k}}\inf  \frac{f(z) - f(x^{k}) - \langle \nu^{k}, z - x^{k}\rangle}{\|z - x^{k}\|} \geq 0,\ \forall k$, and $ x^{k}\stackrel{f}{\longrightarrow} x$ denoting  $x^{k}\rightarrow x$ and $f(x^{k})\rightarrow f(x)$. We also write $\textbf{dom}\partial_{lim} f:= \{ x\in \mathbb{R}^{n}:\ \partial f(x)\neq 0 \}$.
\end{Definition}
It is obvious that the limiting subdifferential coincides with   the gradient for  differentiable functions.
%$f$, which is denoted by $\nabla f$.
Moreover, when $f$ is convex, the limiting subdifferential equal to the subdifferential in convex analysis.
 %The limiting subdifferential enjoys rich and comprehensive calculus rules and has been widely used in nonsmooth and nonconvex optimization.
 Without loss of generality, we use the notation $\partial f$ to denote limiting subdifferential  in the rest of the paper.
Finally, we will say that $x^{\ast}\in \mathbb{R}^{n}$ is a stationary or critical point of $f$ if $0 \in \partial f(x^{\ast})$, which is a necessary condition for $x^{\ast}\in \arg\min_{x} f(x)$.

 Next, we recall the  KL property, KL function and KL exponent which are basic tools   used in convergence analysis for   nonconvex problems. These results are adopted from  \cite{Attouch2010ProximalAM,Attouch2013Convergence,Li2016Calculus,Ochs2014iPiano}.
\begin{Definition}
We say that a proper closed function $f$ has the Kurdyka-{\L}ojasiewicz (KL) property at $ \bar{x}\in \textbf{dom} \partial f $ if there exist a neighborhood $\mathcal{N}$ of $\bar{x}$, $\nu\in(0,\infty]$ and a continuous concave function $\psi:\ [0,\nu)\rightarrow\mathbb{R}_{+}$ with $\psi(0) = 0$ such that:
\begin{enumerate}[i)]
\item $\psi$ is a continuously differentialable on $(0, \nu)$ with $\psi^{\prime}$ over $(0, \nu)$;
\item for all $x\in\mathcal{N}$ with $f(\bar{x}) < f(x) < f(\bar{x}) + \nu$, one has
\begin{equation}
\psi^{\prime}\left( f(x) - f(\bar{x})\right)\text{dist}\left(0, \partial f(x)\right)  \geq 1.
\end{equation}
\end{enumerate}
A proper closed function $f$ satisfying the KL property at all points in $\textbf{dom}\partial f$ is called a KL function.
\end{Definition}

\begin{Definition}
For a proper closed function $f$ satisfying the KL property at $x\in\textbf{dom}\partial f$, if the corresponding function $\psi$ can be chosen as $\psi(s)=\bar{c}s^{1-\alpha}$ for some $\bar{c} > 0$ and $\alpha\in [0,1) $, i.e., there exist $c,\epsilon>0$ and $\nu\in (0, \infty]$ so that
\begin{equation}
\text{dist}(0, \partial f(x))\geq c(f(x) - f(\bar{x}))^{\alpha}
\end{equation}
whenever $\|x-\bar{x}\|\leq \epsilon$ and $f(\bar{x}) < f(x) < f(\bar{x}) + \nu$, then we say that $f$ has the KL property at $\bar{x}$ with an exponent of $\alpha$. If $f$ is a KL function and has the same exponent $\alpha$ at any $\bar{x}\in \textbf{dom}\partial f$, then we say that $f$ is a KL function with a exponent of $\alpha$.
\end{Definition}

\begin{Proposition}\label{klindex}
%{\cite{Li2016Calculus}}
The objective cost function $F$ defined in \eqref{nonconvex model} is a KL function with an exponent of $\frac{1}{2}$.
\end{Proposition}
\begin{proof}
Follows  from \textbf{Corollary 5.2} of \cite{Li2016Calculus}.
\end{proof}

Last, we recall the main conditions to prove the convergence of general  algorithms for nonconvex problems.
%For convience, we list three assumptions firstly.
Let $H:\ \mathbb{R}^{p}\rightarrow\mathbb{R}\cup\{\infty\}$ be a proper lower semi-continuous function and $\{x^{k}\}_{k=0}^{\infty}$ a sequence generated by some optimization method. Assume the following conditions are satisfied:
\begin{itemize}
\item (\textbf{H1})\ The sequence $\{H(x^{k})\}_{k=0}^{\infty}$ is monotonically decreasing thus converging. In particular for any finite starting point $x^{0}\in\mathbb{R}^{p}$, there exists some positive constant $\theta$, such that the sequence $\{x^{k}\}_{0}^{\infty}$ satisfies
\begin{equation}
H(x^{k})-H(x^{k+1})\geq \theta\|x^{k}-x^{k+1}\|^{2}_{2};
\end{equation}
\item (\textbf{H2})\ For each $k\in\mathbb{N}$, there exists some $d^{k+1}\in\partial H(x^{k+1})$, such that
\begin{equation}
\|d^{k+1}\| \leq C\|x^{k+1} - x^{k}\|_{2},
\end{equation}
where $C>0$;
\item (\textbf{H3}) There exists a subsequence $\{x^{k_{l}}\}_{l=0}^{\infty}$ of $\{x^{k}\}_{k=0}^{\infty}$, s.t.,
\begin{equation}
x^{k_{l}}\rightarrow x^{\ast}\ \text{and}\ H(x^{k_{l}})\rightarrow H(x^{\ast});
\end{equation}
\end{itemize}

\begin{Proposition}\label{finite length}
Let $H:\ \mathbb{R}^{p}\rightarrow\mathbb{R}\cup\{\infty\}$ be a proper lower semi-continuous function. Consider a sequence $\{x^{k}\}_{k\in\mathbb{N}}$ that satisfies $(\textbf{H1})-(\textbf{H3})$. If $H$ has the KL property at some cluster point $x^{\ast}\in\mathbb{R}^{p}$ specified in $(\textbf{H3})$, then the sequence $\{x^{k}\}_{k\in\mathbb{N}}$ convergences to $\bar{x} = x^{\ast}$ as $k$ goes to infinity, and $\bar{x}$ is a critical point of $H$. Moreover the sequence $\{x^{k}\}_{k\in\mathbb{N}}$ has a finite length, i.e.
\begin{equation}
\sum_{k}^{\infty}\|x^{k} - x^{k+1}\|<\infty.
\nonumber
\end{equation}
\end{Proposition}
\begin{proof}
Follows  from \textbf{Lemma 2.6} of \cite{Attouch2013Convergence}.
\end{proof}
% \textbf{Proof}:
\subsection{Linear convergence rate }
\begin{Theorem}\label{H123}
Let $\{x^{k} \}_{k=0}^{\infty}$ be the sequence generated  by CD \textbf{Algorithm 1} for objective cost function $F$ defined in (\ref{nonconvex model}) with SCAD or MCP penalty. If the sequences admits  a accumulation point $x^{\ast}$, then
\begin{itemize}
\item (a) (\textbf{H1}) holds, i.e.,
\begin{equation}
F(x^{k})-F(x^{k+1})\geq \theta\|x^{k}-x^{k+1}\|^{2}_{2};
\nonumber
\end{equation}
\item (b) (\textbf{H2}) holds, i.e., for each $k\in\mathbb{N}$, there exists some $d^{k+1}\in\partial F(x^{k+1})$, such that
\begin{equation}
\|d^{k+1}\| \leq C\|x^{k+1} - x^{k}\|_{2};
\nonumber
\end{equation}
\item (c) Let  $x^{k_{l}}\rightarrow x^{\ast}$, then we have
\begin{equation}
F(x^{k_{l}})\rightarrow F(x^{\ast});
\nonumber
\end{equation}
\item (d) $\{x^{k}\}_{k=0}^{\infty}$ converges to  $x^{\ast}$. % is the critical point of $F$;
The sequence $\{x^{k}\}_{k\in\mathbb{N}}$ has a finite length, i.e.
\begin{equation}
\sum_{k}^{\infty}\|x^{k} - x^{k+1}\|<\infty.
\nonumber
\end{equation}
\item (e) $\{x^{k}\}_{k=0}^{\infty}$ converges to  $x^{\ast}$ linearly.
\end{itemize}
\end{Theorem}

\begin{proof}
\begin{enumerate}[(a)]
\item  We use $\rho$ to short for $\rho_{\lambda,\tau}$.
Let $\theta=(1 + \min\{\rho^{\prime\prime}(|t|), 0\})/2.$
For fixed $k,i$, by the definition of CD  \textbf{Algorithm 1},  $$x_{i}^{k+1}= S_{\lambda,\tau}(c_{i}^{k}) \in \arg\min_{t} f_i(t)=\frac{1}{2}(t-c^k_i)^{2}+\rho(t).$$  We can conclude that
\begin{equation}\label{kkt}
0\in (x_{i}^{k+1}-c_i^{k}) + \partial \rho(x_{i}^{k+1}):=  \partial f_{i}(x_{i}^{k+1})
\end{equation}
and
\begin{equation}
\begin{aligned}
& F(x_{0}^{k+1}, ..., x_{i-1}^{k+1}, x_{i}^{k},...,x_{p}^{k+1}) - F(x_{0}^{k+1}, ..., x_{i}^{k+1}, x_{i+1}^{k},...,x_{p}^{k+1}) \\
= & f_i(x_{i}^{k}) - f_i(x_{i}^{k+1}) \\
\geq &\theta |x_{i}^{k} - x_{i}^{k+1}|^{2},
\nonumber
\end{aligned}
\end{equation}
where we use the strong convexity of $f_i$ in the last inequality. Then,
\begin{equation}
F(x^{k})-F(x^{k+1})\geq \theta\|x^{k}-x^{k+1}\|^{2}_{2}
\nonumber
\end{equation}
follows by summarizing the the above display  over all  coordinates. We then easily obtain
\begin{equation}
\sum_{k=1}^{\infty}\|x^{k}-x^{k+1}\|^{2}_{2}\leq F(x^{0})/ \theta <\infty
\nonumber
\end{equation}
and $\lim_{k\rightarrow\infty}\|x^{k}-x^{k+1}\|=0$.
\item Recall that $F(x)=\frac{1}{2}\|Ax - b\|^{2}_{2} + \sum_{i=1}^{p}\rho(x_{i})$. Consider the limiting subdifferential at $x^{k+1}$
\begin{equation*}
\partial F(x^{k+1}) = A^{T}(Ax^{k+1} - b) +
\left(
\begin{array}{ccc}
\partial \rho(x^{k+1}_{1}) \\
\partial \rho(x^{k+1}_{2}) \\
\vdots \\
\partial \rho(x^{k+1}_{p}) \\
\end{array}
\right)   \\
\end{equation*}
By (\ref{kkt}), we have
$- x^{k+1}_{i} + c_{i}^{k}\in \partial\rho(x^{k+1}_{i})$, i.e.,
\begin{align*}
- x^{k+1}_{i} + A_{i}^{T}\left(b - \sum_{j=1}^{i-1}x_{j}^{k+1}A_{j} - \sum_{j=i+1}^{p}x_{j}^{k}A_{j}\right) & \in \partial\rho(x^{k+1}_{i}) \\
%- x^{k+1}_{i} + A_{i}^{T}\left(\sum_{j=i}^{p}x_{j}^{k+1}A_{j} - \sum_{j=i+1}^{p}x_{j}^{k}A_{j}\right) & \in \partial\rho(x^{k+1}_{i}) + A_{i}^{T}\left(\sum_{j=1}^{p}x_{j}^{k+1}A_{j} - b\right) \\
%- x^{k+1}_{i} + A_{i}^{T}\left(\sum_{j=i}^{p}x_{j}^{k+1}A_{j} - \sum_{j=i+1}^{p}x_{j}^{k}A_{j}\right) & \in \partial\rho_{\lambda,\tau}(x^{k+1}_{i}) + A_{i}^{T}\left(Ax^{k+1} - b\right)\\
- x^{k+1}_{i}+ x^{k+1}_{i} + A_{i}^{T}\left[\sum_{j=i+1}^{p}(x_{j}^{k+1} - x_{j}^{k})A_{j}\right] & \in \partial\rho(x^{k+1}_{i}) + A_{i}^{T}\left(Ax^{k+1} - b\right)\\
\sum_{j=i+1}^{p}(x_{j}^{k+1} - x_{j}^{k})A_{i}^{T}A_{j} & \in \partial\rho(x^{k+1}_{i}) + A_{i}^{T}\left(Ax^{k+1} - b\right).
\end{align*}
Let $d^{k+1}_{i}=\sum_{j=i+1}^{p}(x_{j}^{k+1} - x_{j}^{k})A_{i}^{T}A_{j}$. Then the above display shows
 $$d^{k+1} \in \partial F(x^{k+1}).$$  With $|A_{i}^{T}A_{j}|\leq 1$, we can conclude that,
\begin{equation}
|d^{k+1}_{i}|^{2} = \left|\sum_{j=i+1}^{p}(x_{j}^{k+1} - x_{j}^{k})A_{i}^{T}A_{j}\right|^{2} \leq (p-i)\sum_{j=i+1}^{p}\left|x_{j}^{k+1} - x_{j}^{k}\right|^{2} \leq (p-i)\|x^{k+1} - x^{k}\|^{2}.
\nonumber
\end{equation}
Then,
\begin{equation}
\|d^{k+1}\| \leq p\|x^{k+1} - x^{k}\|.
\nonumber
\end{equation}

\item %The convergence of the subsequence means that
%\begin{equation}
%\|x^{k_{l}} - x^{\ast}\| < \epsilon/\ell,\ \forall \epsilon>0,
%\nonumber
%\end{equation}
%where $\ell$ is the Lipschitz constant of $F$. Then, one can easily achieve that
%\begin{equation}
%F(x^{k_{l}}) - F(x^{\ast}) \leq \ell \cdot \|x^{k_{l}} - x^{\ast}\| < \epsilon.
%\nonumber
%\end{equation}
Follows from the continuity of $F$. Moreover, the assumption that there exist a accumulation point  implies
 (\textbf{H3}) holds.
 % assumption. Furthermore, we deduce that $F(x^{k})\rightarrow F(x^{\ast})$ and $F(x^{k})\geq F(x^{\ast})$ for the nonincreasing property of $\{F(x^{k})\}$.
\item This  can be easily verified by applying \textbf{Proposition} \ref{finite length} with KL property.
%\end{proof}
%
%\begin{Theorem}\label{ConvergenceRate}
%Let $\{x^{k} \}_{k=0}^{\infty}$ be the sequence produced by algorithm for objective cost function $F$ defined in \eqref{nonconvex model}, if there exists a accumulation point $x^{\ast}$,  then $x^{k}$ convergences to $x^{\ast}$ linearly.
%\end{Theorem}
%\begin{proof}
\item
By Proposition \ref{klindex}, $F$ admits the KL property with exponent of $1/2$.
Then, using the Definition 2.3  and (b) we have
\begin{equation}
C^{2}\|x^{k} - x^{k+1}\|^{2} \geq \text{dist}^{2}(0, \partial F(x^{k+1}))\geq c^{2}(F(x^{k+1}) - F(x^{\ast})).
\nonumber
\end{equation}
Let $A_{k+1} = \sum_{i=k+1}^{\infty}\|x^{i} - x^{i+1}\|^{2}$.
The above display and  (a) implies,
\begin{equation}
C^{2}(A_{k} - A_{k+1})  \geq \text{dist}^{2}(0, \partial F(x^{k+1}))\geq c^{2}\theta^{2}A_{k+1},
\nonumber
\end{equation}
%and
%\begin{equation}
%M^{2}\geq c^{2}\theta^{2}A_{k+1},
%\nonumber
%\end{equation}
which leads to
\begin{equation}
A_{k+1}\leq\frac{C^{2}}{C^{2} + c^{2}\theta^{2}}A_{k}:=\nu^{2}A_{k},
\nonumber
\end{equation}
where $\nu\in (0,1)$ obviously.
From the above display and the  finite length of $\{x^{k}\}$ in (d), we know that there exists some $\eta>0$, such that
\begin{equation}
A_{k}\leq \nu^{2k}\eta^{2}
\nonumber
\end{equation}
and
\begin{equation}
\|x^{k}-x^{k+1}\|\leq \nu^{k}\eta.
\nonumber
\end{equation}
Then by triangle inequality and the convergence of $\{x_k\}_k$ to $x^*$ we have
\begin{equation}
\|x^{k} - x^{\ast}\| \leq \sum_{i=k}^{\infty}\|x^{i} - x^{i+1}\|\leq \frac{\nu^{k}\eta}{1-\xi},
\nonumber
\end{equation}
i.e,  $x^{k}$ globally converges to $x^{\ast}$  linearly.
\end{enumerate}
\end{proof}
\section{Conclusion}
 In this work, we prove the linear convergence rate of  coordinate descent method for solving MCP/SCAD penalized least squares problems.
 In the proof we use the assumption that  the sequences admits a accumulation point. Removing this assumption is an interesting question for further study.

\bibliography{ref}

\begin{thebibliography}{10}

\bibitem{Attouch2010ProximalAM}
H.~Attouch, J.~Bolte, P.~Redont, and A.~Soubeyran.
\newblock Proximal alternating minimization and projection methods for
  nonconvex problems: An approach based on the kurdyka-lojasiewicz inequality.
\newblock {\em Math. Oper. Res.}, 35:438--457, 2010.

\bibitem{Attouch2013Convergence}
H.~Attouch, J.~Bolte, and B.~F. Svaiter.
\newblock Convergence of descent methods for semi-algebraic and tame problems:
  proximal algorithms, forward-backward splitting, and regularized gauss-seidel
  methods.
\newblock {\em Mathematical Programming}, 137(1-2):91--129, 2013.

\bibitem{breheny2011coordinate}
P.~Breheny and J.~Huang.
\newblock Coordinate descent algorithms for nonconvex penalized regression,
  with applications to biological feature selection.
\newblock {\em The annals of applied statistics}, 5(1):232, 2011.

\bibitem{candes2005decoding}
E.~{Candes} and T.~{Tao}.
\newblock Decoding by linear programming.
\newblock {\em IEEE Transactions on Information Theory}, 51(12):4203--4215,
  2005.

\bibitem{chen2001atomic}
S.~S. {Chen}, D.~L. {Donoho}, and M.~A. {Saunders}.
\newblock Atomic decomposition by basis pursuit.
\newblock {\em Siam Review}, 43(1):129--159, 2001.

\bibitem{fan2001variable}
J.~{Fan} and R.~{Li}.
\newblock Variable selection via nonconcave penalized likelihood and its oracle
  properties.
\newblock {\em Journal of the American Statistical Association},
  96(456):1348--1360, 2001.

\bibitem{fan2004nonconcave}
J.~{Fan} and H.~{Peng}.
\newblock Nonconcave penalized likelihood with a diverging number of
  parameters.
\newblock {\em Annals of Statistics}, 32(3):928--961, 2004.

\bibitem{Li2016Calculus}
G.~Li and T.~K. Pong.
\newblock Calculus of the exponent of kurdyka-{\\l}ojasiewicz inequality and
  its applications to linear convergence of first-order methods.
\newblock {\em Foundations of Computational Mathematics}, pages 1--34, 2016.

\bibitem{mazumder2011sparsenet}
R.~Mazumder, J.~H. Friedman, and T.~Hastie.
\newblock Sparsenet: Coordinate descent with nonconvex penalties.
\newblock {\em Journal of the American Statistical Association},
  106(495):1125--1138, 2011.

\bibitem{meinshausen2006high}
N.~{Meinshausen} and P.~{Bühlmann}.
\newblock High-dimensional graphs and variable selection with the lasso.
\newblock {\em Annals of Statistics}, 34(3):1436--1462, 2006.

\bibitem{Ochs2014iPiano}
P.~Ochs, Y.~Chen, T.~Brox, and T.~Pock.
\newblock ipiano: Inertial proximal algorithm for non-convex optimization.
\newblock {\em Siam Journal on Imaging Sciences}, 7(2):1388--1419, 2014.

\bibitem{1998Variational}
T.~R. Rockafellar and J.~B. Wets.
\newblock Variational analysis.
\newblock {\em in Sobolev and BV Spaces, MPS-SIAM Series on Optimization},
  30:324--326, 1998.

\bibitem{she2009thresholding}
Y.~She et~al.
\newblock Thresholding-based iterative selection procedures for model selection
  and shrinkage.
\newblock {\em Electronic Journal of statistics}, 3:384--415, 2009.

\bibitem{tibshirani1996regression}
R.~{Tibshirani}.
\newblock Regression shrinkage and selection via the lasso.
\newblock {\em Journal of the royal statistical society series
  b-methodological}, 58(1):267--288, 1996.

\bibitem{tropp2010computational}
J.~A. {Tropp} and S.~J. {Wright}.
\newblock Computational methods for sparse solution of linear inverse problems.
\newblock {\em Proceedings of the IEEE}, 98(6):948--958, 2010.

\bibitem{tseng2001convergence}
P.~Tseng.
\newblock Convergence of a block coordinate descent method for
  nondifferentiable minimization.
\newblock {\em Journal of optimization theory and applications},
  109(3):475--494, 2001.

\bibitem{zhang2010Nearly}
C.-H. Zhang et~al.
\newblock Nearly unbiased variable selection under minimax concave penalty.
\newblock {\em The Annals of statistics}, 38(2):894--942, 2010.

\bibitem{zhang2008the}
C.~H. {Zhang} and J.~{Huang}.
\newblock The sparsity and bias of the lasso selection in high-dimensional
  linear regression.
\newblock {\em Annals of Statistics}, 36(4):1567--1594, 2008.

\bibitem{zhang2010analysis}
T.~Zhang.
\newblock Analysis of multi-stage convex relaxation for sparse regularization.
\newblock {\em Journal of Machine Learning Research}, 11(3), 2010.

\bibitem{zhao2006on}
P.~{Zhao} and B.~{Yu}.
\newblock On model selection consistency of lasso.
\newblock {\em Journal of Machine Learning Research}, 7(90):2541--2563, 2006.

\bibitem{zou2008one}
H.~Zou and R.~Li.
\newblock One-step sparse estimates in nonconcave penalized likelihood models.
\newblock {\em Annals of statistics}, 36(4):1509, 2008.

\end{thebibliography}

%\begin{thebibliography}{99}% Reference smaple:  http://www.ybook.co.jp/pjo-aim.pdf
%\bibitem{}
%
%\bibitem{}
%\end{thebibliography}
\end{document}